\documentclass[runningheads]{llncs}

\usepackage{eccvabbrv}

\usepackage{graphicx}
\usepackage{booktabs}
\usepackage{subcaption}

\usepackage[accsupp]{axessibility}  

\usepackage{hyperref}

\usepackage{orcidlink}

\usepackage{multirow} 
\usepackage{tabularx}
\usepackage{arydshln}
\usepackage{amsmath}
\usepackage{xcolor}
\usepackage{amssymb}
\usepackage{makecell}

\newcommand{\jy}[1]{\textcolor{black}{#1}}
\newcommand{\wh}[1]{\textcolor{black}{#1}}

\begin{document}

\title{Adaptive Context Matters: Towards Provable Multi-Modality Guidance for Super-Resolution} 

\titlerunning{Towards Provable Multi-Modality Guidance for Super-Resolution}

\author{Jinyi Luo \and
Minghao Liu \and
Yifan Li \and
Zejia Fan \and
Jiaying Liu
}

\authorrunning{J. Luo, M. Liu \textit{et al.}}

\institute{Wangxuan Institute of Computer Technology, Peking University}

\maketitle

\begin{abstract}
 Super-resolution (SR) is a severely ill-posed problem with inherent ambiguity, as widely recognized in both empirical and theoretical studies. 
 Although recent semantic-guided and multi-modal SR methods exploit large models or external priors to enhance semantic alignment, the fusion of heterogeneous modalities remains insufficiently understood in practice and theory. 
 In this work, we provide the first theoretical modeling of multi-modal SR, revealing that prior methods are bottlenecked by sub-optimal modality utilization. 
 Our analysis shows that the generalization risk bound can be improved by strengthening the alignment between modality weights and their effective contributions, while reducing representation complexity. 
 This theoretical insight inspires us to propose the novel \textbf{M}ulti-\textbf{M}odal \textbf{M}ixture-of-\textbf{E}xperts \textbf{S}uper-\textbf{R}esolution framework (\textbf{M$^3$ESR}) that employs generalization-oriented dynamic modality fusion for accurate risk control and modality contribution optimization. In detail, we propose a novel spatially dynamic modality weighting module and a temporally adaptive modality temperature scheduling mechanism, enabling flexible and adaptive spatial-temporal modality weighting for effective risk control. Extensive experiments demonstrate that our M$^3$ESR significantly boosts generalization and semantic consistency performances, which confirms our superiority.
  \keywords{
  Multi-Modality Guidance \and Super-Resolution \and Dynamic Fusion \and Uncertainty \and Mixture-of-Experts }
\end{abstract}

\section{Introduction}
\label{sec:intro}

Single image super-resolution (SISR) seeks to reconstruct high-resolution \jy{(HR)} images from \jy{low-resolution (LR)} observations obtained under complex real-world imaging processes, in which downsampling is intertwined with degradations including blur, noise, compression artifacts, \textit{etc}.
In recent years, substantial progress has been achieved in this field, evolving from traditional reconstruction-based methods~\cite{Kim2010single,Xiong2010Robust,Freedman2011ImageAV} and early deep learning-based approaches~\cite{srcnn,fsrcnn,swinir,ipt} to recent techniques that exploit generative model priors~\cite{realesrgan,resshift,diffbir,stablesr}, reflecting a continuing evolution toward improved perceptual quality.
%
Nevertheless, due to the ill-posed nature of super-resolution, methods relying solely on low-resolution inputs often face severe ambiguity, leading to reconstructions that may deviate from plausible semantic structures.

\begin{figure}[t]
    \centering
    \includegraphics[width=\linewidth]{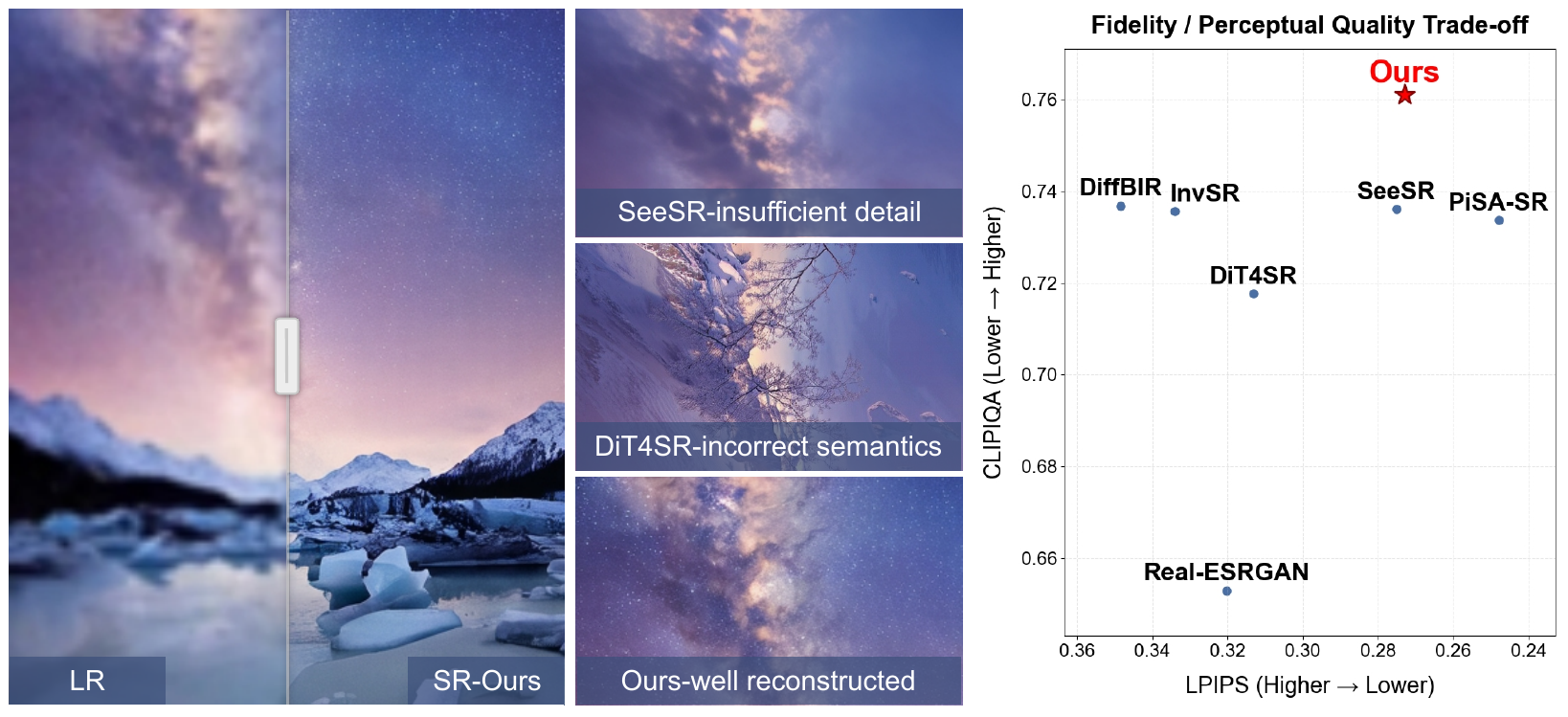}
    \caption{
    Example qualitative and quantitative comparisons that demonstrate the superior semantic fidelity and perceptual quality achieved by our proposed M$^3$ESR method. While existing SR methods either struggle to generate sufficient details or introduce semantic inconsistencies, our method effectively addresses both through the incorporation of multi-modal guidance signals.
    }
    \label{fig:teaser}
    \vspace{-5mm}
\end{figure}

To tackle this issue, recent works~\cite{seesr,tvt,dit4sr} introduce text guidance with complementary semantic information into SR. They incorporate pretrained visual encoders and text-to-image generative models with sources of textual semantic guidance, thereby enhancing perceptual quality through improved alignment with high-level semantics.
However, due to the abstract and global (non-pixelwise) nature of textual signals, semantic-guided SR methods remain limited in providing fine-grained semantic cues, often leading to misalignment between guided semantics and local image regions.
%

A recent study on semantic-guided super-resolution, MMSR~\cite{mmsr}, explores the integration of multi-modal guidance for SR reconstruction. It incorporates multiple complementary modalities derived from the low-resolution input, including semantic segmentation, edges, and depth, within a unified feature encoding framework. This design enables more fine-grained and semantically consistent guidance, leading to improved reconstruction quality.
%
However, while improving semantic perceptual performance, this strategy may introduce potential generalization risks. Under real-world degradations and content variations, semantic errors in extracted modality cues can be propagated through entangled multi-modal representations, resulting in unreliable reconstructions.


The fundamental issue lies in the lack of explicit regulation of modality-related risks in multi-modal super-resolution. Existing multi-modal guidance-based SR methods extract multiple modality cues from low-resolution inputs and fuse them into a unified representation, implicitly assuming that the guidance signals are reliable and complementary, while overlooking semantic errors from LR information loss and modality redundancy.
Our theoretical analysis reveals that, without explicit risk regulation, risky modality cues can be entangled with LR features, leading to increased generalization risk and unreliable reconstruction.

In detail, we reformulate multi-modal-guided super-resolution in the framework of generalization risk minimization. 
%
We derive an upper bound on the generalization risk of multi-modal-guided super-resolution, consisting of a static fusion baseline, training data size and two high-level components:
\textbf{Modality Contribution-Correlation}, which characterizes the alignment between the weights assigned to each modality and the marginal performance gain contributed by that modality; and
\textbf{Fusion Complexity}, which captures the change in the model’s Rademacher complexity relative to the static fusion baseline.
%
We show that, under a Mixture-of-Experts (MoE) framework, positively correlating modality weights with their marginal contributions tightens the generalization bound, enabling improved generalization via dynamic fusion.

Inspired by this insight, we propose a \textbf{Multi-Modal Mixture-of-Experts Super-Resolution} framework (\textbf{M$^3$ESR}) that employs spatial–temporal modality risk estimation and dynamic weighting for compact multi-modal aggregation.
Specifically, we leverage the uncertainty map as a region-wise modality risk indicator and learn an uncertainty-based router to assign patch-specific weights to different experts, enabling spatially dynamic, gain-aware modality control.
%
Meanwhile, we introduce a timestep-aware modality temperature scheduler that adaptively adjusts each expert’s attention temperature across diffusion steps, enabling stage-dependent modulation of modality importance.
Together, our framework dynamically selects and adjusts the modality guidance signal with the highest marginal benefit at each spatial region and diffusion timestep.

Our contributions can be summarized as follows:
\begin{itemize}
    \item We propose a generalization-oriented theory for multi-modal SR, showing that generalization risk is governed by modality weight-contribution correlation and representation complexity, which in turn inspires a compact multi-modal guidance extraction-injection design.
    \item We design a risk-aware spatial dynamic modality weighting mechanism that uses uncertainty maps to capture region-wise modality contributions and adaptively modulate multi-modal weights for generalization risk reduction.
    \item We introduce a timestep-aware modality scheduling mechanism that learns timestep-based temperature schedules to adaptively regulate modality expert attention. It effectively handles the time-varying reliance on different semantic modalities during the SR process.
\end{itemize}

\section{Related Work}
\subsection{Semantic-guided Image Super-resolution}
Early deep-learning-based single-image super-resolution methods utilize convolutional neural network (CNNs)~\cite{srcnn,fsrcnn,msrn} and Transformer-based architectures~\cite{swinir,ipt,pftsr} for global and regional information extraction to fit LR-HR pairs generated through simple bicubic downsample degradation, and struggle to handle real complex degradation patterns in real-world scenarios. 
More recent methods leverage the priors from generative models on SR task to better model real-world degradation, such as BSRGAN~\cite{bsrgan}, RealESRGAN~\cite{realesrgan}, Reshift~\cite{resshift}, StableSR~\cite{stablesr}, and DiffBIR~\cite{diffbir}.
However, due to the ill-posed nature of the SISR task, reconstruction solely based on LR inputs often suffers from semantic ambiguity caused by information loss, leading to results that differ from human semantic perception. 

To address this issue, a group of works~\cite{seesr,dit4sr,tvt,pasd} incorporate semantic-level guidance signals into the SR process, such as leveraging text-prompts or semantic-related neural information. 
PromptSR~\cite{promptsr} perceives degradation-related semantic guidance using a pretrained multi-modal large language model. 
SeeSR~\cite{seesr} builds a degradation-aware prompt extractor to extract text-style hard and embedding-style soft semantic prompts for SR guidance. 
PiSA-SR~\cite{pisasr} disentangles semantic-level detail generation from pixel-level reconstruction with a separate LoRA module.
PURE~\cite{pure} proposes a mixed high-level understanding-low-level restoration scheme in the same auto-regressive generation process.
DiT4SR~\cite{dit4sr} leverages the attention space of Multi-Modal Diffusion Transformers (MMDiT) for more effective text semantic information injection.
Although these methods successfully improve human-perceived semantic quality by incorporating global semantic guidance, such guidance is mainly provided in the form of global textual embeddings. As a result, mismatches may occur between specific semantic instances and image regions, limiting the effectiveness of fine-grained semantic guidance.

\subsection{Multi-modal Signal-guided Image Super-resolution}

To provide spatial guidance, many recent super-resolution methods have attempted to incorporate spatially related modalities beyond textual prompts.
SegSR~\cite{segsr} utilizes segmentation maps to help identify and better restore salient semantic components in the image.
%
RAGSR~\cite{ragsr} performs bounding box detection and aligns text with image regions via a regional attention mechanism. However, these methods primarily focus on spatial semantic alignment and do not further explore the role of diverse modality guidance in super-resolution.
UPSR~\cite{upsr} employs uncertainty maps to regulate the noise level in the diffusion process, thereby controlling super-resolution behavior across different texture regions. However, it does not incorporate semantic information, nor does it explore effective coordination between semantic cues and texture reconstruction.
MMSR~\cite{mmsr} incorporates additional modality signals, including semantic segmentation, edge, and depth maps, through a multi-modal latent connector. Nevertheless, it assumes uniformly positive contributions from all modalities, lacking sample-specific modality regulation to suppress modality-induced risks dynamically. 
%
In this work, we explore multi-modal semantic guidance for super-resolution and introduce dynamic mechanisms to reduce modality-induced risks and enhance modality-specific contributions.

\section{Generalization-Guided Risk-Aware Multi-Modal SR }


In this section, we conduct a theoretical analysis on generalization risk reduction, which reveals its dependence on marginal modality weight correlation and modal representation complexity, and then design a compact representation injection framework with dynamic weighting and scheduling mechanism based on these insights.

\subsection{Multi-Modal SR Generalization Bound Minimization}

Our multi-modal-guided SISR problem can be formally defined as follows. Given an input LR image $x \in \mathcal{X} = \mathbb{R}^{H \times W \times C}$ and $M$ guidance signals from different modalities $\mathbf{m} = \{\textbf{m}_1, \ldots, \textbf{m}_M\}$, where each modality feature is represented as $\textbf{m}_i \in \mathbb{R}^{N_i \times d}$, our goal is to learn an SR model $f$ that minimizes the expected reconstruction loss between the predicted SR image $f(x,\mathbf{m})$ and the target HR image $y \in \mathcal{Y} = \mathbb{R}^{(sH) \times (sW) \times C}$:
\begin{equation}
\operatorname*{arg\,min}_{f} \;  \; \mathbb{E}_{(x,y)\sim\mathcal{D}}
\big[ \ell\big(f(x,\mathbf{m}), y\big) \big],
\end{equation}
where $\mathcal{D}$ denotes the unknown real-world distribution over $\mathcal{X} \times \mathcal{Y}$, 
$\ell: \mathcal{Y} \times \mathcal{Y} \rightarrow \mathbb{R}_+$
is the SR loss function, which is assumed to be Lipschitz continuous
with respect to its first argument (\jy{\emph{e.g.}}, the $L_2$ loss), $s$ denotes the super-resolution scaling factor,
\jy{$H$ and $W$ are image height and width}, $C$ is the number of image channels,
$N_i$ denotes the number of tokens in the $i$-th modality,
and $d$ is the unified feature dimension of the multi-modal guidance.

To instantiate the multi-modal-guided SR model $f$, we build a Mixture-of-Experts framework with intermediate feature fusion design, which consists of an MoE guidance feature fusion layer $h_{\text{MoE}}: \mathcal{X} \times \mathcal{M} \to \mathbb{R}^{N_p \times d_h}$ and corresponding post-process layers $g: \mathbb{R}^{N_p \times d_h} \to \mathcal{Y}$. Here $\mathcal{M}$ denotes the space of multi-modal guidance signals, $N_p$ is the number of patches in the patchified fused feature and $d_h$ is the hidden dimension of the fused feature. Thus, the model can be built as $f(x, \mathbf{m}) = g \circ h_{\text{MoE}}(x, \mathbf{m})$ with:
\begin{equation}
    h_{\text{MoE}}(x, \mathbf{m}) = z_x + \sum_{m=1}^{M} \sum_{i=1}^{N_p} w^m_i \cdot E_m(z_{x,i}, \textbf{m}_m),
\end{equation}
where $z_x = \text{VAE-Enc}(x) \in \mathbb{R}^{N_p \times d_h}$ is the patch embeddings of the LR input $x$, $z_{x,i} \in \mathbb{R}^{d_h}$ refers to the embedding of the $i^{th}$ patch, \jy{$\circ$ denotes function compostion,} $E_m: \mathbb{R}^{d_h} \times \mathbb{R}^{L_m \times d} \to \mathbb{R}^{d_h}$ is the $m^{th}$ expert, and $w_i^m$ is the weight assigned to the output on the $i^{th}$ patch from the $m^{th}$ expert.

\wh{In} real-world scenarios, the degradation of LR inputs is \wh{sample-dependent and highly diverse}, which leads to \wh{sample-dependent usefulness} of modality guidance. Specifically, the \wh{effectiveness} of modality guidance may vary significantly across different samples, exhibiting \wh{different levels of relevance} and suffering from modality information errors as well as modality-LR region mismatches.
Under a static fusion framework, modality weights are treated as constants for each modality across all samples and spatial regions, \wh{lacking the ability to adapt to sample-wise variations in modality guidance quality}. Consequently, \wh{as shown in Fig.~\ref{fig:wrong_guidance}, when modality guidance is unreliable}, misleading modality information can be introduced into the super-resolution results, \wh{potentially causing significant degradation and} limiting generalization capability across scenarios.

\begin{figure}[t]
    \centering
    \includegraphics[width=0.8\linewidth]{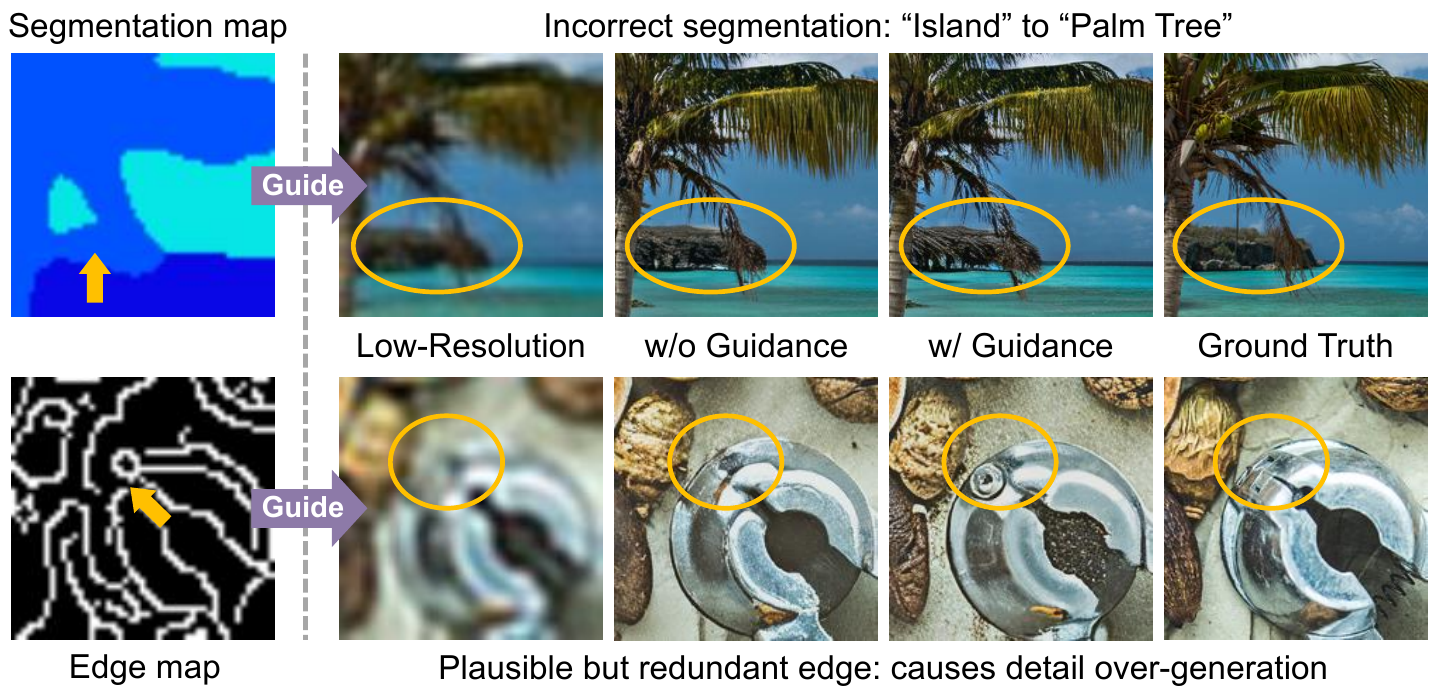}
    \vspace{-2mm}
    \caption{\textbf{SR reconstruction results are degraded by inaccurate guidance signals}.
Under generalized LR degradations, modality guidance estimation inevitably suffers from errors (\jy{\emph{e.g.}}, misclassified regions in segmentation maps) and redundancy (\jy{\emph{e.g.}}, uninformative regions in depth maps). Such inaccurate guidance, when fused in a static manner, significantly degrades reconstruction performance.}
    \label{fig:wrong_guidance}
    \vspace{-6mm}
\end{figure}


To investigate this problem, we consider the generalization error of a multi-modal SR model on the real-world degradation domain as $\mathbb{E}_\mathcal{D}[\ell(f(x), y)]$. According to Rademacher Complexity Theory~\cite{rademacher}, the following expansion holds with probability at least $1-\delta$:
\begin{equation}
\label{eq:radmacher_expansion_f}
    \mathbb{E}_{\mathcal{D}}[\ell(f(x), y)] \leq \hat{\mathbb{E}}_{\mathcal{S}}[\ell(f(x), y)] + 2L\mathcal{R}_N(\mathcal{F}) + \sqrt{\frac{\ln(1/\delta)}{2N}},
\end{equation}
where $\hat{\mathbb{E}}_{\mathcal{S}}$ is the empirical expectation, $\mathcal{S}$ is the training data domain, $\mathcal{R}_N(\mathcal{F})$ is the Rademacher complexity, $\mathcal{F}$ is the function space of SR models, $L$ is a constant, and $N$ is the training data size. We compare the generalization bound of dynamic fusion with static fusion $f_{\text{static}}$. 
%
%
We present a main theorem showing that the expected generalization loss of dynamic fusion models is no greater than static fusion, with the improvement governed by the correlation between expert weights and marginal modality gains as well as the increase in model complexity.

\begin{theorem}[Generalization Advantage of MoE Dynamic Fusion]
Let $f_{\text{MoE}}: \mathcal{X} \times \mathcal{M} \to \mathcal{Y}$ be the multi-modal SR model based on MoE with  dynamic fusion design, $f_{\text{static}}: \mathcal{X} \times \mathcal{M} \to \mathcal{Y}$ be the static fusion baseline, then for $\forall \delta \in (0,1) $, the following generalization bound holds with probability at least $1-\delta$:
\begin{equation}
\boxed{
\begin{aligned}
\sup_{f\sim \mathcal{F}_{\text{MoE}}}
\mathbb{E}_{\mathcal{D}}
\bigl[\ell(f(x,\mathbf{m}), y)\bigr]
\;\le\;
&\sup_{f\sim \mathcal{F}_{\text{static}}}
\mathbb{E}_{\mathcal{D}}
\bigl[\ell(f(x,\mathbf{m}), y)\bigr] \\
&- \Gamma_{\text{MoE}}
+ \Delta_\mathcal{R}(\mathcal{F}_{\text{MoE}},\mathcal{F}_{\text{static}})
+ \mathcal{O}\left(\frac{1}{\sqrt{N}}\right).
\end{aligned}
}
\end{equation}
\end{theorem}

Here, \( \mathcal{F}_{\text{MoE}} \) and \( \mathcal{F}_{\text{static}} \) denote the function
spaces corresponding to the dynamic MoE model and the static fusion baseline,
respectively. The terms \( \Gamma_{\text{MoE}} \) and
\( \Delta_{\mathcal{R}}\!\left(\mathcal{F}_{\text{MoE}}, \mathcal{F}_{\text{static}}\right) \)
represent the covariance between the marginal gain induced by MoE and the assigned modality weights, and the additional model
complexity introduced by MoE, respectively, which are defined as:
\begin{align}
    & \Gamma_{\text{MoE}} := \sum_{m=1}^{M} \sum_{i=1}^{N_p} \text{Cov}\left(w^m_i, \Delta^m_i(x, y)\right), \\
    & \Delta_\mathcal{R}(\mathcal{F}_{\text{MoE}},\mathcal{F}_{\text{static}}) := 2L(\mathcal{R}_N(\mathcal{F}_{MoE})-\mathcal{R}_N(\mathcal{F}_{static})).
\end{align}

Consequently, under identical settings, the generalization bound of dynamic fusion MoE can be optimized by maximizing \( \Gamma_{\text{MoE}} \) while reducing the model complexity.

We provide a brief proof of the proposed theorem on the generalization advantage of dynamic MoE fusion below. 

\begin{proof}

Based on the Rademacher expansions in Ineq.~\ref{eq:radmacher_expansion_f}, 
the difference between the upper bounds of the expected generalization error of dynamic MoE and static models is given by:
\begin{equation}
\label{eq:radmacher_expansion}
\begin{aligned}
    \mathbb{E}_{\mathcal{D}}
\bigl[\phi(f_{\text{MoE}}) - \phi(f_{\text{static}})\bigr]
 &\leq  \;\hat{\mathbb{E}}_{\mathcal{S}}
\bigl[\phi(f_{\text{MoE}}) - \phi(f_{\text{static}})\bigr] \\
&+ \;\Delta_\mathcal{R}(\mathcal{F}_{\text{MoE}},\mathcal{F}_{\text{static}}) + \mathcal{O}\left(\frac{1}{\sqrt{N}}\right).
\end{aligned}
\end{equation}

 We define the marginal contribution of modality $m$ on the $i^{th}$ patch as 
$$\Delta^m_i = - \langle \nabla_z \phi(z_{\text{static}}), E_m(z_{x,i}, \textbf{m}_m) \rangle.$$ \jy{Let $\bar{w}^m$ denotes the sample-invariant expert weight for the $m$-th modality in static fusion models,} the loss difference expectation can be written as:
\begin{equation}
\begin{aligned}
&\hat{\mathbb{E}}_{\mathcal{S}}[\ell(f_{\text{MoE}}(x, \mathbf{m}), y)] - \hat{\mathbb{E}}_{\mathcal{S}}[\ell(f_{\text{static}}(x, \mathbf{m}), y)] \\
=&  - \hat{\mathbb{E}}_{\mathcal{S}}\left[\sum_{m=1}^{M} \sum_{i=1}^{N_p} (w^m_i - \bar{w}^m) \cdot \Delta^m_i(x,y)\right] + \hat{\mathbb{E}}_{\mathcal{S}}[R(x,y)], \label{eq:loss_expectation_expansion}
\end{aligned}
\end{equation}
where \( R(x,y) := \frac{1}{2} \langle \Delta z, H_\phi(z_{\text{static}}) \Delta z \rangle  \) is a sample-dependent higher-order remainder term, and we can guarantee that $|\mathbb{E}[R(x,y)]| \leq C \cdot M \cdot N_p = O(1)$ holds on the training set.
%
%
For the marginal contribution term, we have:
\begin{equation}
\begin{aligned}
\mathbb{E}[(w^m_i - \bar{w}^m) \Delta^m_i(x,y)] 
= (\mathbb{E}[w^m_i] - \bar{w}^m) \mathbb{E}[\Delta^m_i] + \text{Cov}(w^m_i, \Delta^m_i).
\end{aligned}
\end{equation}

Apply the unbiasedness constraint of modality weights over $\mathcal{S}$, we have:
\begin{equation}
    \mathbb{E}[(w^m_i - \bar{w}^m) \Delta^m_i] = \text{Cov}(w^m_i, \Delta^m_i).
\end{equation}

Thus, Eq.~\ref{eq:loss_expectation_expansion} is simplified as:
\begin{equation}
\begin{aligned}
   \hat{\mathbb{E}}_{\mathcal{S}}[\ell(f_{\text{MoE}}, y)] -  \hat{\mathbb{E}}_{\mathcal{S}}[\ell(f_{\text{static}}, y)] 
    =  - \sum_{m=1}^{M} \sum_{i=1}^{N_p} \text{Cov}(w^m_i, \Delta^m_i) + o(N^{-1/2}).
\end{aligned}
\end{equation}

Substituting the above expression into Ineq.~\ref{eq:radmacher_expansion} gives:
\begin{equation}
    \begin{aligned}
        &{\mathbb{E}}_{\mathcal{D}}[\ell(f_{\text{MoE}}, y)] -  {\mathbb{E}}_{\mathcal{D}}[\ell(f_{\text{static}}, y)] \\
        \leq & - \Gamma_{\text{MoE}} + \Delta_\mathcal{R}(\mathcal{F}_{\text{MoE}},\mathcal{F}_{\text{static}}) + \mathcal{O}\left(\frac{1}{\sqrt{N}}\right).
    \end{aligned}
\end{equation}

Thus, the Dynamic Fusion MoE Generalization Advantage Theorem is proved.
\end{proof}

\begin{figure}[t]
    \centering
    \includegraphics[width=\linewidth]{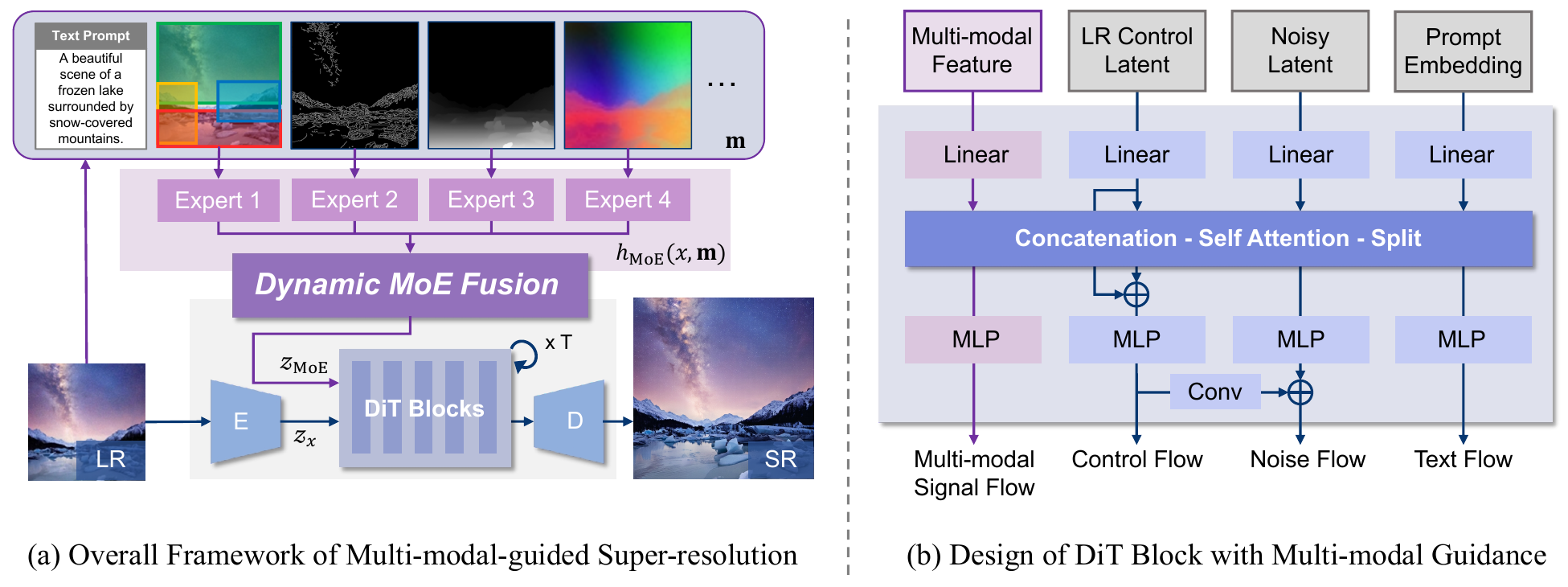}
    \caption{The Multi-Modal Mixture-of-Experts Super-Resolution framework. (a) Multi-modal guidance signals are extracted from the LR input and fused into a compact representation via dynamic MoE fusion to guide the diffusion-based SR process. (b) \jy{Compact multi-modal representation guidance is injected through DiT self-attention.}}
    \label{fig:framework}
    \vspace{-5mm}
\end{figure}

\subsection{Compact Multi-Modal Representation Injection}

Based on the above understanding that a dynamic fusion Mixture-of-Experts design can effectively bound the generalization risk, we propose a Multi-Modal Mixture-of-Experts Super-Resolution (M$^3$ESR) framework. 
%
The core of this framework lies in \wh{a compact multi-modal representation extraction and injection mechanism inspired by the above analysis,} built upon the MoE paradigm, \wh{through which we design an adaptive fusion strategy that} enables efficient and controlled utilization of heterogeneous guidance signals.

Specifically, beyond the original text prompt used in semantic-guided SR models, we introduce four additional guidance modalities: semantic segmentation, depth, edge, and DINO features.
%
We fuse these modality-specific guidance into a unified and compact multi-modal feature by a MoE-based dynamic fusion module.
The fused multi-modal feature is concatenated with the noisy latent, LR latent, and text embeddings, then jointly processed through self-attention for semantic-aware denoising. Under a Diffusion Transformer (DiT) backbone, the attention outputs are propagated across DiT blocks to form an information flow, and the final noise prediction corresponds to a single diffusion denoising step, which is decoded to obtain the final HR image.
Detailed design of our modality feature injection approach is depicted in Fig.~\ref{fig:framework}.

This compact modality feature design aligns naturally with our generalization risk minimization objective for multi-modal-guided SR, \wh{bringing three advantages}:
First, compared with direct attention-based fusion between multi-modal features and LR latent, the compact design \wh{extracts more intrinsic modality information while removing redundant and potentially conflicting signals and making the guidance more lightweight}.
Second, by aggregating homogeneous guidance signals, the compact feature extraction process \wh{filters out stochastic variations and sample-specific noise}, thereby lowering model complexity and improving generalization.
Finally, the compact modality feature acts solely as an informative initialization for the DiT backbone, \wh{decoupling modality utilization from subsequent backbone learning}, avoiding repeated modality interaction modeling across layers and \wh{enabling more efficient modality utilization}.

\subsection{{Risk-Aware Dynamic Modality Weighting}}



We design a dynamic fusion mechanism that adaptively aligns modality weights with the marginal gains introduced by different guidance signals. 
However, explicitly estimating modality-specific contributions in SR is challenging due to \wh{(1) unobservable modality contributions} and \wh{(2) spatially varying guidance quality}.

We observe that modality contribution strongly correlates with \wh{local texture characteristics} of the target HR image.
\wh{Smooth regions favor structural cues (\textit{e.g.}, depth), while texture-rich regions benefit more from contour or edge guidance.}
Such patterns \wh{and their guidance preferences} usually can be captured by an uncertainty map $\textbf{U}(x)$ computed from the LR images $x$.
%
%

\begin{figure}[t]
    \centering
    \includegraphics[width=0.85\linewidth]{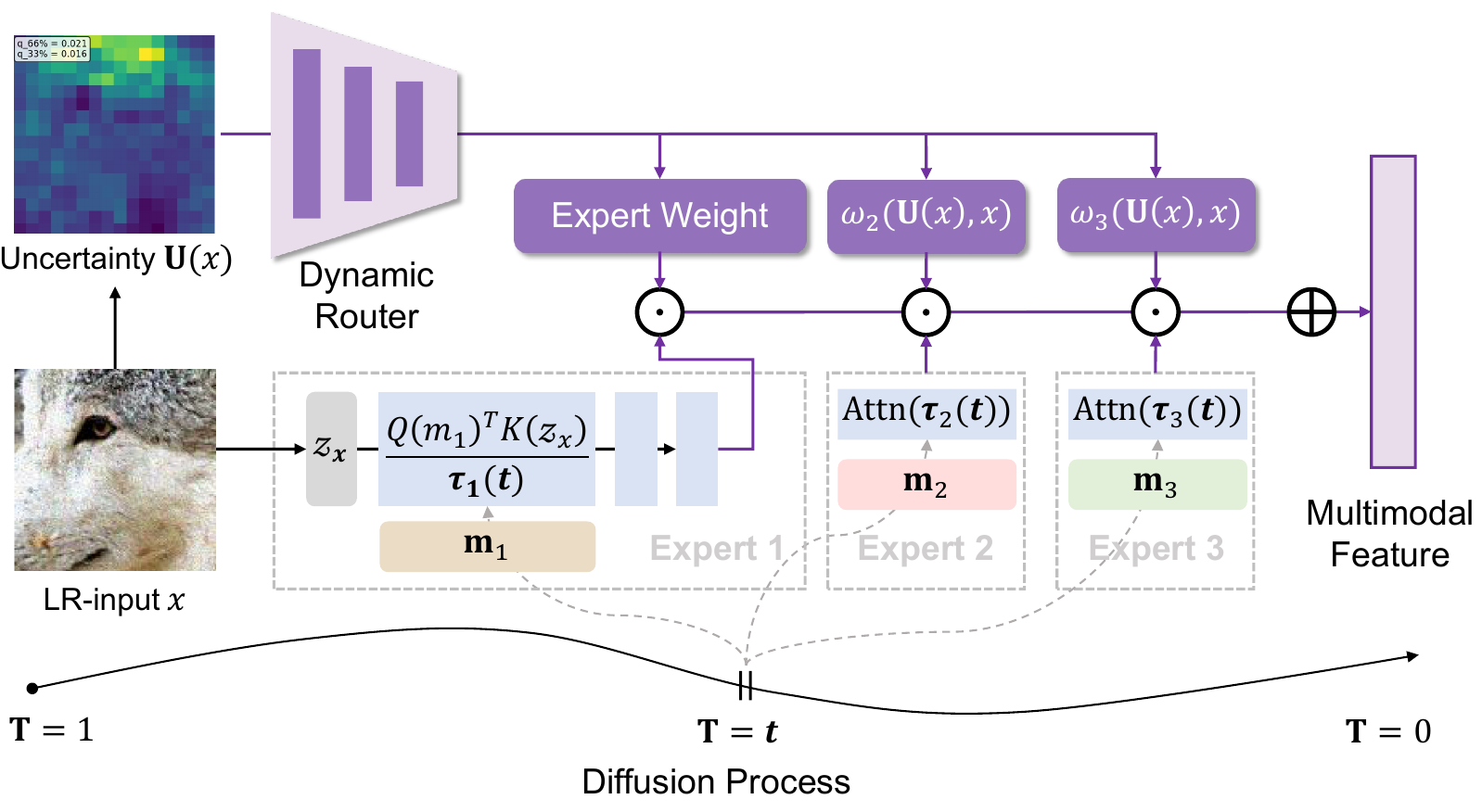}
    \vspace{-2mm}
    \caption{The architecture of our dynamic fusion module. For spatial modality contribution enhancement, we leverage uncertainty maps as cues for patch-wise expert weight assignment via a learnable dynamic router. For temporal modality risk reduction, we adopt timestep-aware attention temperature to adjust modality guidance signals.}
    \label{fig:fusion_module}
    \vspace{-6mm}
\end{figure}

Based on such observations, we design a {risk-aware dynamic modality weighting} mechanism by utilizing a region-specific modality router conditioned on the uncertainty map to leverage the local contribution and risk introduced by different modality guidance signals for each expert.
To be specific, as shown in Fig.~\ref{fig:fusion_module}, 
each modality expert $E_m$ is a lightweight cross-attention–based module operating on the LR latent $z_x$ and modality latent $\textbf{m}_m$. We first calculate the uncertainty map $\textbf{U}(x) \in \mathbb{R}^{N_p}$ based on the difference between the LR input $x$ and a coarse SR result given by a generative model RealESRGAN~\cite{realesrgan}:
\begin{equation}
    \textbf{U}(x) = ||\tilde{f}(x) - \text{Up-Scale}(x)||_1,
\end{equation}
where $\tilde{f}(x)$ is the coarse SR model.
%
We train a region-specific router based on the acquired uncertainty values, which is a Transformer-based network that accepts an LR image and uncertainty map as input, and outputs region-specific modality weights $w^m_\theta(\textbf{U}(x),x) \in [0,1]^{N_p} $. These weights are used to perform element-wise multiplication on the output features of each expert:
\begin{equation}
    h_{\text{MoE}}(x, \mathbf{m}) = z_x + \sum_{m=1}^{M} w^m_\theta(\textbf{U}(x),x) \odot E_m(z_{x}, m_m).
\end{equation}

This uncertainty-based dynamic weighting mechanism effectively enhances the correlation between modality weights and their marginal gains within the MoE framework.
On the one hand, the intrinsic correlation among the uncertainty map, target texture patterns, and modality contributions enables the uncertainty-based router to suppress high-risk or less informative modalities, while allocating higher weights to modalities with significant positive contributions.
On the other hand, the region-aware weighting design allows for \wh{pixel-level} spatial \wh{adjustment}, ensuring that the covariance
$\mathrm{Cov}\!\left(w_i^m, \Delta_i^m(x, y)\right)$
is optimized at the patch level.
\vspace{-3mm}


\subsection{Timestep-Aware Modality Temperature Scheduling}


Beyond spatially varying modality gains, we observe that modality contributions also evolve across diffusion timesteps. Early stages emphasize global structure and semantic reconstruction, \wh{relying more on} high-level guidance (\jy{\emph{e.g.}}, segmentation), whereas later stages focus on fine-grained texture recovery, where edge or depth cues become more beneficial.

Motivated by this observation, we introduce a timestep-aware modality temperature scheduling mechanism that adaptively adjusts the attention temperature of each modality expert. At each timestep, the temperature is designed to be modulated in accordance with the modality contribution: modalities providing reliable guidance are assigned lower temperatures to yield sharper, more concentrated attention distributions, while less reliable modalities receive higher temperatures to produce smoother attention and mitigate potential risk.

The detailed design of this approach is illustrated in Fig.~\ref{fig:fusion_module}. Specifically, we adopt learnable timestep-aware modality-specific temperature curves parameterized by $\alpha_m(t)$ and $\beta_m(t)$, which are learnable time embedding functions of timestep $t$. These curves compute the temperature $\mathbf{\tau}_m(t)$ for each expert at different timesteps:

\begin{equation}
    \begin{aligned}
        &\text{Cross-Attn}(z_x,m_m;t) = \text{SoftMax}\left( \frac{Q(m_m)K(z_x)}{\mathbf{\tau}_m(t)} \right)V(z_x),\\
        &\mathbf{\tau}_m(t) = t(\alpha_m(t) e^{-\beta_m(t)}) + (1-t)(1-\alpha_m(t) e^{-\beta_m(t)}).
    \end{aligned}
\end{equation}

This mechanism is \wh{further augmented into a temporal dynamic version of} the modality weighting strategy.
It enables time-dependent adjustment of modality-specific guidance representations along the diffusion process, making each modality better aligned with the SR objective at the current timestep.

\section{Experiments}

\subsection{Experimental Settings}
\noindent \textbf{Datasets.} In our experiments, we follow the setting in SeeSR [34], training on a combination of LSDIR~\cite{lsdir} and the first 10k FFHQ~\cite{ffhq} images. All images are randomly cropped to $512 \times 512$ patches and downsampled $4\times$ via the Real-ESRGAN~\cite{realesrgan} degradation pipeline.  
%
We evaluate methods on DIV2K-val~\cite{div2k}(synthetic) and RealLQ250~\cite{reallq250} datasets.


\noindent \textbf{Evaluation Metrics.} We adopt 6 evaluation metrics in total, consisting of 1 full-reference (FR) metrics including  LPIPS~\cite{lpips}, with 5 no-reference (NR) metrics including LIQE~\cite{liqe},  WaDIQaM-NR~\cite{wadiqam}, MUSIQ~\cite{musiq}, CLIPIQA~\cite{clipiqa} and HyperIQA~\cite{hyperiqa} for quantitative evaluation. 

\noindent \textbf{Implementation Details.} 
%
We build our multi-modal guided super-resolution framework upon DiT4SR~\cite{dit4sr} with a Stable-Diffusion (SD) 3.5 backbone~\cite{stablediffusion3}, which serves as the primary quality-oriented model in our experiments. For comparison, we additionally implement a fidelity-driven variant on SeeSR~\cite{seesr} with a SD 2.1 backbone.
For modality guidance estimation, the text prompts are generated by LLAVA~\cite{llava} and embedded with CLIP~\cite{clip}.  Semantic segmentation maps are given by Mask2Former~\cite{mask2former}, depth maps are estimated by DepthAnything-V2~\cite{depth_anything_v2}, edge maps are calculated using Canny edge detector~\cite{cannyedge}, and DINO features are calculated using DINO-V3~\cite{dinov3}. The former three modalities are encoded by a VQGAN~\cite{vqgan} encoder.
During inference, we use 40-step diffusion to generate the SR result. During training, we train our fusion model for 250k iterations with batch size 16, starting from a pretrained text-guided SR checkpoint.




\begin{table}[t]
    \centering
     \scriptsize
     \renewcommand{\arraystretch}{1.35}
    \caption{Quantitative comparison results on DIV2K~\cite{div2k} dataset. Best and second-best results are in \textbf{bold} and \underline{underlined}.}
    \label{tab:metric_div2k}
        \begin{tabularx}{\linewidth}{
        c | c |
        >{\centering\arraybackslash}X
        >{\centering\arraybackslash}X
        >{\centering\arraybackslash}X
        >{\centering\arraybackslash}X
        >{\centering\arraybackslash}X
        >{\centering\arraybackslash}X
    }
\hline
\multicolumn{2}{c|}{\textbf{Method}}  & LPIPS$\downarrow$ & LIQE$\uparrow$ & WaDIQam$\uparrow$ & MUSIQ$\uparrow$ & CLIPIQA$\uparrow$  & HyperIQA$\uparrow$ \\ 
\hline
\multirow{2}{*}{\makecell{\textit{GAN-}\\ \textit{based}}}  
&BSRGAN~\cite{bsrgan}  & 0.3192 & {4.0155} & -0.1794 & 66.98 & 0.5924 &  0.6132\\
&R-ESRGAN~\cite{realesrgan}  & 0.3203 & 4.1617 & -0.1634  & 68.47 & 0.6529 & 0.5991 \\
\hline
\multirow{7}{*}{\makecell{\textit{Diffusion-}\\ \textit{based}}}  
&SeeSR~\cite{seesr}  & 0.2751 & \underline{4.6037} & 0.0609  & 72.24 & 0.7361 & \underline{0.7073} \\
&DiffBIR~\cite{diffbir}  & 0.3483 & 4.5559 & -0.0098 & 72.44 & 0.7368 & 0.6697 \\
&PiSA-SR~\cite{pisasr}  & \textbf{0.2479} & 4.5994 & 0.0397 & \underline{73.23} & 0.7336 & 0.7001\\
&InvSR~\cite{invsr} & 0.3339 & 4.0425 & 0.0208 & 71.86 & 0.7356  & 0.6629 \\
&DiT4SR~\cite{dit4sr}  & 0.3132 & 4.5917 & 0.0152 & 73.05 & 0.7176 & 0.6953 \\
\cline{2-8}
&Ours-F    & \underline{0.2652} & 4.5682 & \underline{0.0810} & 72.02 & \underline{0.7396} & 0.7064 \\
&Ours-Q   & 0.2729 & \textbf{4.6899} & \textbf{0.0895} & \textbf{73.29} & \textbf{0.7610} & \textbf{0.7153} \\
\hline
\end{tabularx}
\end{table}

\subsection{Quantitative Comparison Analysis}
We compare our M$^3$ESR method with two GAN-based methods including BSRGAN~\cite{bsrgan} and Real-ESRGAN~\cite{realesrgan}, and five state-of-the-art Diffusion-based methods, including SeeSR~\cite{seesr}, PiSA-SR~\cite{pisasr}, InvSR~\cite{invsr}, DiffBIR~\cite{diffbir} and DiT4SR~\cite{dit4sr}. 
Given the strong data distribution fitting capability of diffusion-based methods, we further fine-tune the released checkpoints of InvSR, DiT4SR, and DiffBIR for the same number of epochs to ensure a fair comparison despite differences in their original training settings.
\jy{We do not compare with MMSR~\cite{mmsr} as its code is not publicly available and its reported results use a different training setting.}
%

As shown in Tab.~\ref{tab:metric_div2k}, the fidelity-oriented variant (Ours-F) achieves FR metrics comparable to state-of-the-art diffusion methods, while the quality-oriented variant (Ours-Q) outperforms other SR methods on NR metrics, including method like DiT4SR that uses the same SD 3.5 backbone.
%
%
%
It is worth noting that existing image quality assessment metrics struggle to capture the semantic consistency of fine-grained details. Smooth SR textures may achieve favorable PSNR and LPIPS scores yet appear perceptually inconsistent to humans, while semantically plausible fine-grained textures often deviate from the ground truth at the pixel level, resulting in lower FR metric scores.

\begin{figure}[t!]
    \centering
    \includegraphics[width=\linewidth]{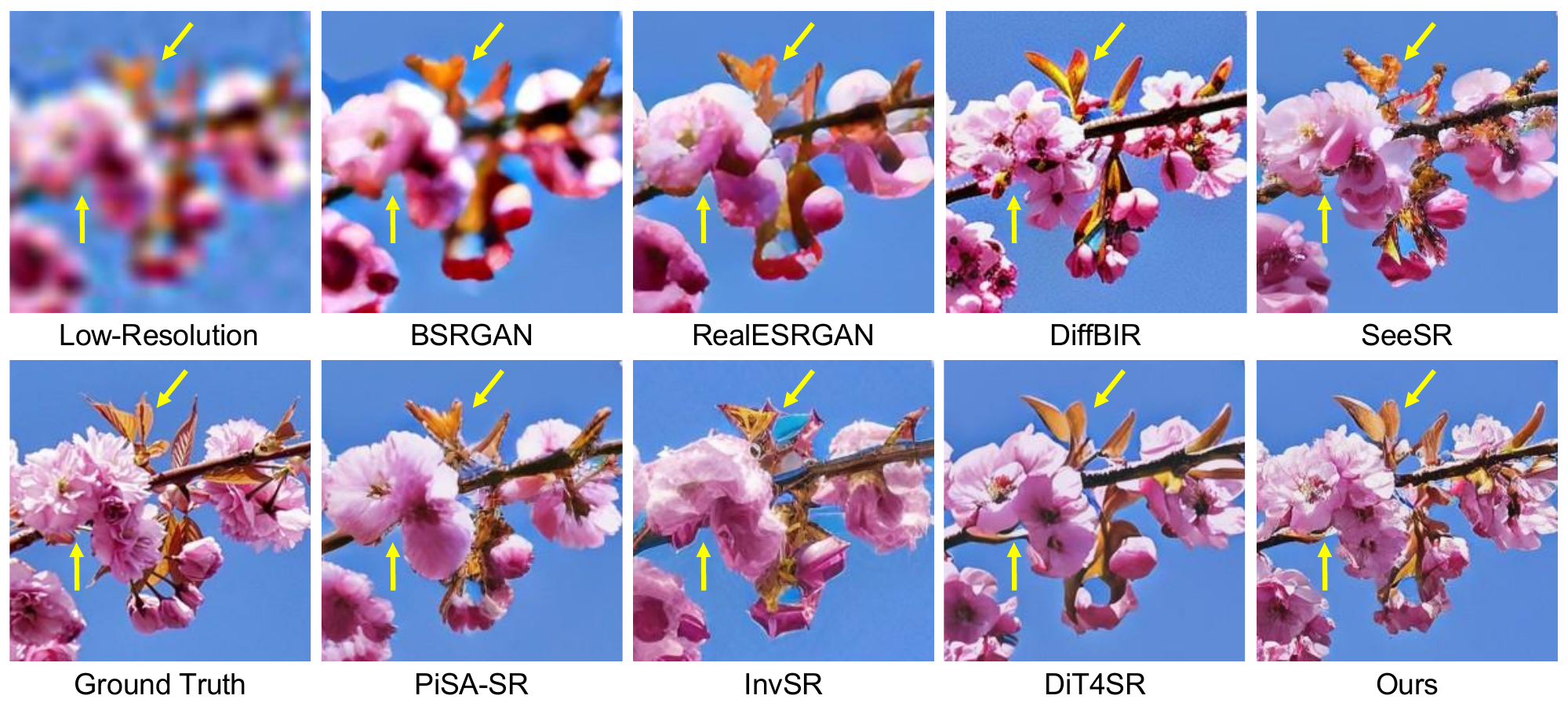}
    \caption{Visualization results of our M$^3$ESR method compared with other approaches on DIV2K~\cite{div2k} dataset.}
    \label{fig:vis_div2k}
\end{figure}

\begin{figure}[t!]
    \centering
    \includegraphics[width=\linewidth]{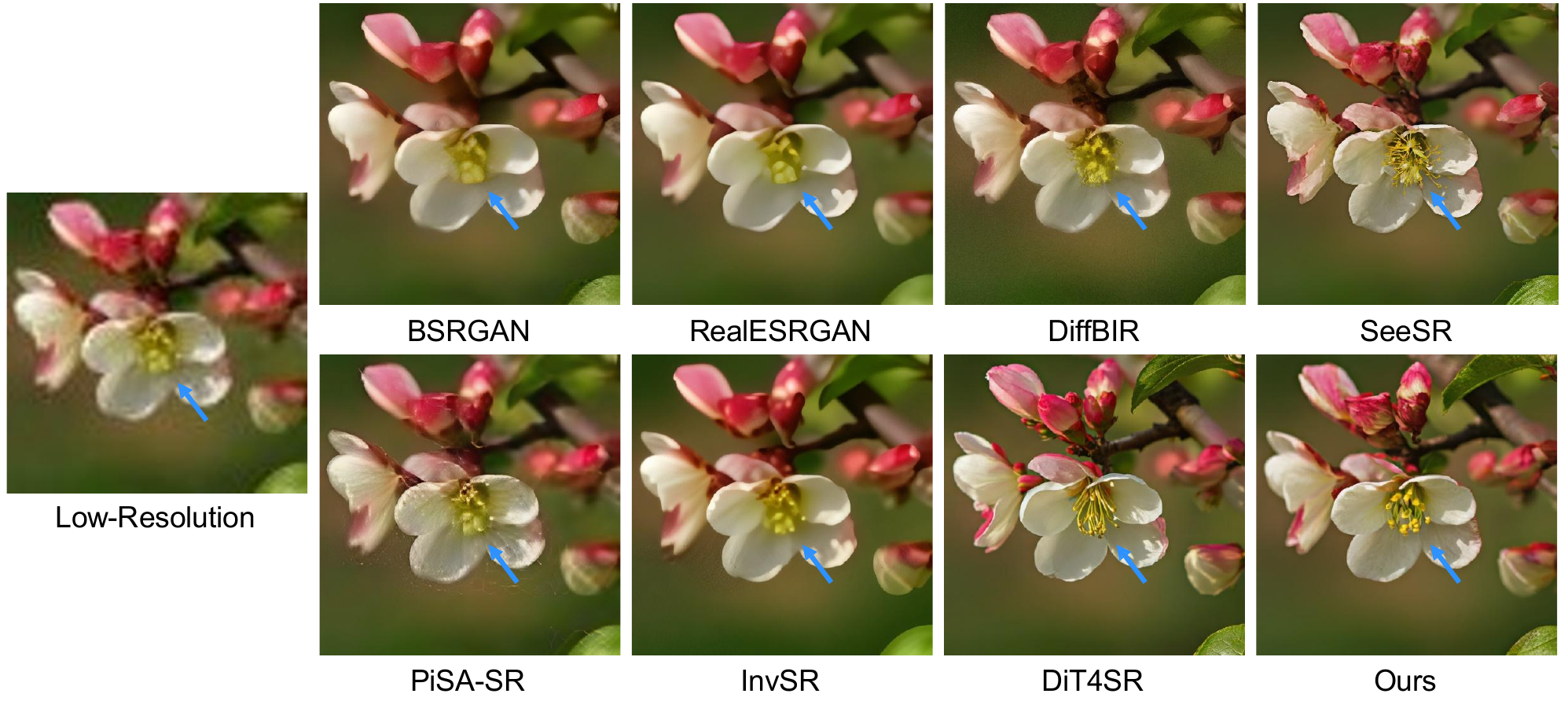}
    \caption{Visualization results of our M$^3$ESR method compared with other approaches on RealLQ250~\cite{reallq250} dataset.}
    \label{fig:vis_reallq}
\end{figure}

\begin{table}[t!]
    \centering
    \scriptsize
    \renewcommand{\arraystretch}{1.35}
    \caption{User study results: Best-selection percentage across three aspects.}
    \begin{tabularx}{\textwidth}{
        l |
        >{\centering\arraybackslash}X
        >{\centering\arraybackslash}X
        >{\centering\arraybackslash}X
        >{\centering\arraybackslash}X
    }
    \hline
        \textbf{Method} & SeeSR~\cite{seesr} & PiSA-SR~\cite{pisasr} & DiT4SR~\cite{dit4sr} & Ours-Q \\
    \hline
        Semantic consistency & 11.39\% & 18.33\% & 13.61\% & \textbf{56.67}\% \\
        Perceptual quality & 20.83\% & 11.94\% & 13.33\% & \textbf{53.89}\% \\
        Pixel fidelity & 9.72\% & 23.61\% & 6.94\% & \textbf{59.72}\% \\
        \hline
    \end{tabularx}
    \label{tab:user_study}
\end{table}

Hence, in order to better evaluate semantic restoration performance from a human perspective, we conducted a user study on DIV2K~\cite{div2k} dataset and no-reference RealLQ250~\cite{reallq250} dataset. The evaluation was carried out from three aspects: LR-SR semantic consistency, subjective perceptual quality, and pixel-level fidelity(only on DIV2K). For each case, every participant was asked to select the best-performing result for each evaluation aspect. A total of 24  participant responses were collected across 15 cases. The percentages of selections as the best-performing method, reported in Tab.~\ref{tab:user_study}, indicate that our approach achieves favorable performance under human perception.

\subsection{Qualitative Comparison Analysis}
\label{subsec:visualize_results}

We present qualitative comparisons on the DIV2K dataset in Fig.~\ref{fig:vis_div2k} and on RealLQ250~\cite{reallq250} dataset in Fig.~\ref{fig:vis_reallq} .
Compared with other methods, our M$^3$ESR approach produces more delicate details and achieves superior perceptual quality, while maintaining faithful semantic consistency. For example, in the first case, our result produces richer petal textures while preserving the semantic correctness of the leaves and flowers, indicating that the proposed multi-modal guidance steers generative detail synthesis toward semantically faithful reconstruction.
%

\subsection{Ablation Studies}

\begin{table*}[t]
    \centering
    \scriptsize
    \renewcommand{\arraystretch}{1.35}
    \caption{Ablation results on DIV2K dataset.}
    
    \begin{subtable}[t]{0.48\linewidth}
        \centering
        \caption{Contribution of each modality.}
        \begin{tabularx}{\linewidth}{
            >{\raggedright\arraybackslash}X |
            >{\centering\arraybackslash}X
            >{\centering\arraybackslash}X
            >{\centering\arraybackslash}X
        }
        \hline
        Modalities & LPIPS$\downarrow$ & LIQE$\uparrow$ & MUSIQ$\uparrow$ \\
        \hline
        None      & 0.3132 & 4.5917 & 73.05 \\
        w/o seg.  & 0.2796 & 4.6118 & 72.54 \\
        w/o depth & 0.2931 & 4.6336 & 73.08 \\
        w/o edge  & \underline{0.2793} & 4.6809 & \underline{73.18} \\
        w/o DINO  & 0.2797 & \textbf{4.7251} & 73.07 \\
        ALL       & \textbf{0.2729} & \underline{4.6899} & \textbf{73.29} \\
        \hline
        \end{tabularx}
        \label{tab:ablation_modality}
    \end{subtable}
    \hfill
    \begin{subtable}[t]{0.48\linewidth}
        \centering
        \caption{MoE fusion design.}
        \begin{tabularx}{\linewidth}{
            c c |
            >{\centering\arraybackslash}X
            >{\centering\arraybackslash}X
            >{\centering\arraybackslash}X
        }
        \hline
        Dynamic & Temp. & \multirow{2}{*}{LPIPS$\downarrow$} &
\multirow{2}{*}{LIQE$\uparrow$} &
\multirow{2}{*}{MUSIQ$\uparrow$} \\
        Routing & Scheduling & & & \\
        \hline
        $\checkmark$ & $\times$      & 0.2906 & 4.5942 & 72.33 \\
        $\times$     & $\checkmark$ & 0.2780 & 4.6347 & 72.77 \\
        $\checkmark$ & $\checkmark$ & \textbf{0.2729} & \textbf{4.6899} & \textbf{73.29} \\
        \hline
        \end{tabularx}
        \label{tab:ablation_module}
    \end{subtable}

\end{table*}

We validate the effectiveness of our multi-modal mixture-of-experts design through thorough ablation studies.
Specifically, we conduct evaluations from two complementary perspectives: (i) modality contribution, which examines the marginal performance contribution of each guidance modality, and (ii) the Mixture-of-Experts module effectiveness, which analyzes the impact of different modules in the expert feature extracting and routing mechanism.
\vspace{2mm}

\noindent \textbf{Modality Contribution}. Excluding the textual modality, we incorporate four guidance modalities: semantic segmentation maps, edges, depth maps, and DINO features.
We perform ablation studies by removing each modality during training to examine its marginal contributions to the SR results.
As shown in Tab.~\ref{tab:ablation_modality}, removing any individual modality degrades the FR metric, represented by LPIPS, and the overall performance drops significantly when no modality is used. This indicates that the proposed dynamic fusion effectively leverages complementary modality information despite the inaccuracy in single modality guidance.
\vspace{2mm}

\noindent \textbf{MoE Module Effectiveness}. We remove the uncertainty-guided dynamic expert routing module and the timestep-aware expert temperature scheduling mechanism from our MoE framework. As shown in Tab.~\ref{tab:ablation_module}, disabling either component leads to a noticeable drop in overall performance, while combining both modules yields better performance by jointly adjusting expert weights and features. These results demonstrate the effectiveness of our framework design.



\section{Conclusion}
In this work, we propose a multi-modality-guided Super-resolution framework which dynamically utilizes multi-modal information for SR from a generalization risk-minimization perspective. By introducing risk-aware dynamic modality weighting under a mixture-of-experts formulation, together with adaptive modality temperature scheduling, our method jointly reduces the expected generalization loss. Extensive experiments demonstrate the effectiveness of the proposed framework and suggest a possible direction for calibrating the  contributions of guidance modalities in broader low-level vision tasks.


%
%
\bibliographystyle{splncs04}
\bibliography{main}
\end{document}